\documentclass[10pt,twocolumn,letterpaper]{article}

\usepackage{iccv}
\usepackage{times}
\usepackage{epsfig}
\usepackage{graphicx}
\usepackage{amsmath}
\usepackage{amsthm}
\usepackage{amssymb}
\usepackage{microtype}
\usepackage{subfigure}
\usepackage{booktabs}
\usepackage{enumitem}
\usepackage[normalem]{ulem}
\usepackage{color}
\usepackage{sidecap}
\usepackage[accsupp]{axessibility}  %

\usepackage{sidecap,wasysym,array,tabularx,caption}
\definecolor{atomictangerine}{rgb}{0.8, 0.2, 0.1}
\definecolor{turq}{rgb}{0.0, 0.5, 0.5}
\definecolor{darkturq}{rgb}{0.0, 0.4, 0.4}
\definecolor{bright}{rgb}{0.8, 0.1, 0}
\definecolor{darkgray}{gray}{0.3}
\definecolor{mahogany}{rgb}{0.6, 0.05, 0.05}
\definecolor{myblue}{rgb}{0.3,0.05,0.9}

\definecolor{olive}{rgb}{0.537, 0.627, 0.318}
\definecolor{green}{rgb}{0.22, 0.463, 0.114}
\definecolor{grey}{rgb}{0.4, 0.4, 0.4}
\definecolor{blue}{rgb}{0.435, 0.659, 0.863}
\definecolor{pink}{rgb}{0.761, 0.482, 0.627}
\definecolor{darkpink}{rgb}{0.561, 0.282, 0.427}

\newtheorem{theorem}{Theorem}

\usepackage{algorithmic}

\newcommand\ignore[1]{}
\newcommand\DRO{{DRO}}

\newcommand{\epsc}{2\varepsilon_c}
\newcommand{\ecentroid}{\widehat{\mu_c}}
\newcommand{\centroid}{\mu_c}
\newcommand{\dist}{d}
\newcommand{\distz}{\dist(\centroid,z_i)}
\newcommand{\distzp}{\dist(\centroid,z')}
\newcommand{\edistz}{\dist(\ecentroid,z)}
\newcommand{\edistzp}{\dist(\ecentroid,z')}

\newcommand{\qdistz}{\dist(\centroid^q,z)}
\newcommand{\qdistzp}{\dist(\centroid^q,z')}
\newcommand\mypar[1]{\textbf{ #1 }}

\usepackage[breaklinks=true,bookmarks=false]{hyperref}

\iccvfinalcopy %

\def\iccvPaperID{8563} %

\ificcvfinal\fi

\begin{document}

\title{Distributional Robustness Loss for Long-tail Learning}

\author{Dvir Samuel\\
    Bar-Ilan University, Israel\\
    {\tt\small dvirsamuel@gmail.com}
    \and
    Gal Chechik\\
    Bar-Ilan University, Israel\\
    NVIDIA Research, Israel\\
    {\tt\small gal.chechik@biu.ac.il}
}

\maketitle
\ificcvfinal\thispagestyle{empty}\fi

\begin{abstract}
Real-world data is often unbalanced and long-tailed, but deep models struggle to recognize rare classes in the presence of frequent classes. To address unbalanced data, most studies try balancing the data, the loss, or the classifier to reduce classification bias towards head classes. Far less attention has been given to the latent representations learned with unbalanced data. We show that the feature extractor part of deep networks suffers greatly from this bias. We propose a new loss based on robustness theory, which encourages the model to learn high-quality representations for both head and tail classes. While the general form of the robustness loss may be hard to compute, we further derive an easy-to-compute upper bound that can be minimized efficiently. This procedure reduces representation bias towards head classes in the feature space and achieves new SOTA results on CIFAR100-LT, ImageNet-LT, and iNaturalist long-tail benchmarks. We find that training with robustness increases recognition accuracy of tail classes while largely maintaining the accuracy of head classes. The new robustness loss can be combined with various classifier balancing techniques and can be applied to representations at several layers of the deep model. 
\end{abstract}

\section{Introduction}
Real-world data typically has a long-tailed distribution over semantic classes: few classes are highly frequent, while many classes are only rarely encountered. When trained with such unbalanced data, deep models 
tend to produce predictions that are biased and over-confident towards head classes and fail to recognize tail classes. 

Early approaches for handling unbalanced data used resampling \cite{Drummond2003C4,Han2005BorderlineSMOTEAN} or loss reweighing \cite{He_2009_IEEE,Lin2017FocalLF} aiming to re-balance the training process. 
Other approaches address unbalanced data by transferring information from head to tail classes \cite{liu2019large,zhou2020bbn, wang2019dynamic,liu2020deep}, or by applying an adaptive loss \cite{cao2019learning} or regularization of classifiers \cite{Kang2019DecouplingRA}. The main focus of these techniques is on balancing the multi-class classifier. 

\begin{figure}
    \begin{center}
        \includegraphics[width=0.99\columnwidth]{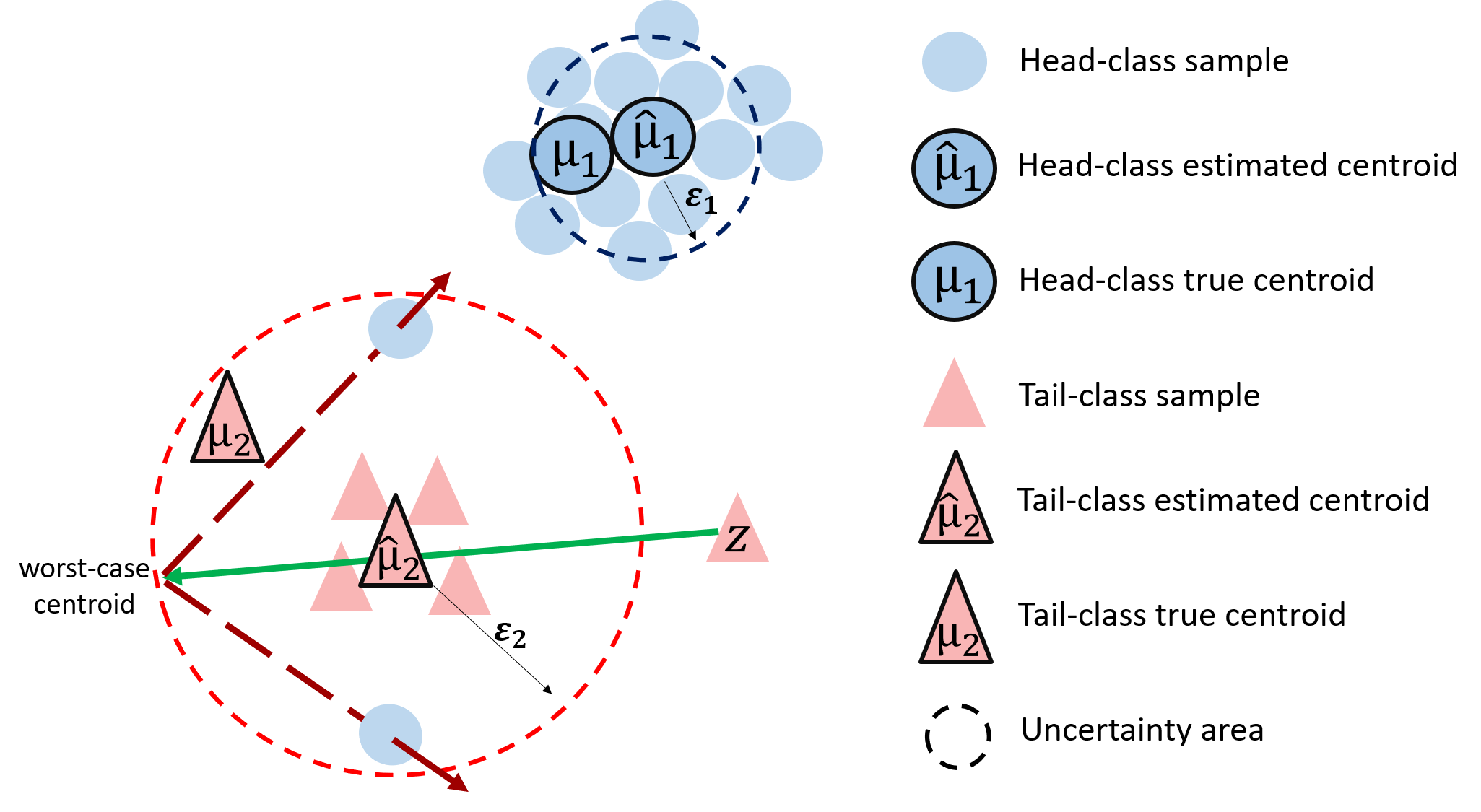}
    \end{center}
    \caption{
    Our \textit{distributional robustness loss} is designed for learning a representation where samples are kept close to the centroid of their class. Here, the empirical centroid $\widehat{\mu_2}$ (framed pink triangle) is estimated based only on few samples (four pink triangles) and as a result it deviates far from the true centroid ($\mu_2$). Our loss pulls the same-class samples (green arrow), and pushes away other-class samples (red arrows). 
    The loss takes into account the estimation error by pushing and pulling towards a worst-case possible distribution within an uncertainty area around the estimated centroid (dashed red line). Uncertainty areas are typically larger for tail classes, compared with head classes that have many samples (blue dashed line around $\widehat{\mu_1}$).
}
    \label{fig:method}
\end{figure}

Far less attention has been given to the latent representations learned with unbalanced data. Intuitively, head classes are encountered more often during training and are expected to dominate the latent representation at the top layer of a deep model. Counter to this intuition,  \cite{Kang2019DecouplingRA} compared a series of techniques for rebalancing representations and concluded that \textit{"data imbalance might not be an issue in learning high-quality representations"}, and that strong long-tailed recognition can be achieved by only adjusting the classifier. However, the effect of unbalanced data on the learned representations is far from being understood, and the extent to which it may hurt classification accuracy is unknown. While existing rebalancing methods do not improve representations, the question remains if better representations could substantially improve recognition with unbalanced data.  

The current paper focuses on improving the representation layer of deep models trained with unbalanced data. We show that large gains in accuracy can be achieved simply by balancing the representation at the last layer. The key insight is that the distribution of training samples of tail classes does not represent well the true distribution of the data. This yields representations that hinder the classifier applied to that representation.

To address this problem, we introduce ideas from robust optimization theory. We design a loss that is more robust to the uncertainty and variability of tail representations. Standard training follows Empirical Risk Minimization (ERM), which is designed to learn models that perform well on the training distribution. However, ERM assumes that the test distribution is the same as the training distribution, and this assumption often breaks with tail classes. In comparison, Distributionally Robust Optimization (DRO) \cite{DRO,GohDRO,bertsDRO} is designed to be robust against likely shifts of the test distribution. It learns classifiers that can handle the worst-case distribution within a neighborhood of the training distribution. Figure 1 illustrates this idea. 

In the general case, computing a loss based on a worst-case distribution may be computationally hard. Here, we show how the worst-case loss can be bounded from above, with a bound that has an intuitive form and can be easily minimized. The resulting bound  allows us to minimize the DRO loss, which only affects the representation, and to combine it with a standard classification loss, which tunes a classifier on top of the representation.

~\newline
The main contributions of this paper are:
\begin{enumerate}[noitemsep,nolistsep]
  \item We formulate learning with unbalanced data as a problem of robust optimization and highlight the role of high variance in hindering tail accuracy.
  \item We develop a new loss, \textbf{DRO-LT Loss}, based on distributional robustness optimization for learning a balanced feature extractor. Training with DRO-LT yields representations that capture well both head and tail classes. 
  \item We derive an upper bound of the DRO-LT loss that can be computed and optimized efficiently. We further show how to learn the robustness safety margin for each class, jointly with the task, avoiding additional hyper-parameter tuning.
  \item We evaluate our approach on four long-tailed visual recognition benchmarks: CIFAR100-LT, CIFAR10-LT, ImageNet-LT, and iNaturalist. Our proposed method consistently achieve superior performance over previous models.
\end{enumerate}

\section{Related Work}

\subsection{Long-tail recognition}
Real-world data usually follows a long-tailed distribution, which causes models to favor head classes and overfit tail classes ~\cite{buda2017systematic, samuel2020longtail}.
Previous efforts to address this effect can be divided into four main approaches: Data-manipulation approaches, Loss-manipulation approaches, Two-stage fine-tuning, and ensemble methods.

\textbf{Data-manipulation (re-sampling):} Data-manipulation approaches aim to balance long-tail data. There are three popular techniques of resampling strategies:
(1) Over-sampling minority (tail) classes by simply copying samples \cite{chawla2002smote, Han2005BorderlineSMOTEAN}. (2) Under-sampling majority (head) classes by removing samples \cite{Drummond2003C4,hu2020learning}. (3) Generating augmented samples to supplement tail classes \cite{Beery2019SyntheticEI, chu2020feature,Kim2020M2mIC,chawla2002smote}.
While simple and intuitive, \textit{over-sampling} methods suffer from heavy over-fitting on the tail classes, \textit{under-sampling} methods degrade the generalization of models, and \textit{data augmentation} methods are expensive to develop.

\textbf{Loss-manipulation (re-weighting):}
Loss reweighting approaches encourage learning of tail classes by setting costs to be non-uniform across classes. For instance, \cite{He_2009_IEEE} scaled the loss by inverse class frequency. An alternative strategy down-weighs the loss of well-classified examples, preventing easy negatives from dominating the loss \cite{Lin2017FocalLF} or dynamically rescale the cross-entropy loss based on the difficulty to classify a sample \cite{Ryou_2019_ICCV}. 
\cite{cao2019learning} proposed to encourage larger margins for rare classes. \cite{khan2019striking} use class-uncertainty information, using Bayesian uncertainty estimates, to learn robust features, and \cite{Cui2019ClassBalancedLB} reweighs using the effective number of samples instead of proportional frequency.

\textbf{Two-stage fine-tuning:}
Two-stage methods separate the training process into representation learning and classifier learning \cite{Ouyang2016FactorsIF, Kang2019DecouplingRA, openlongtailrecognition, samuel2020longtail}. The first stage aims to learn good representations from unmodified long-tailed data, training using conventional cross-entropy without re-sampling or re-weighting.  The second stage aims to balance the classifier by freezing the backbone and finetuning the last fully connected layers with re-sampling techniques or by learning to unbias the confidence of the classifier. Those methods basically assume that the bias towards head classes is significant only in the classifier (i.e, last fully connected layer), or that tweaking the classifier layer can correct the underlying biases in the representation.

\textbf{Ensemble-models:}
Ensemble models focus on generating a balanced model by assembling and grouping models. Typically, classes are separated into groups, where classes that contain similar training instances are grouped together. Then, individual models focused on each group are assembled to form a multi-expert framework. \cite{zhou2020bbn} learned one branch for head classes and another for tail classes, then combine the branches using a soft-fusion procedure.
\cite{LFME} distilled a unified model from multiple teacher classifiers. Each classifier focuses on the classification of a small and relatively balanced group of classes from the data.
\cite{ride} described a shared architecture for multiple classifiers, a distribution-aware loss, and an expert routing module. The current paper proposes a two-stage approach for training one model with a single classifier. Thus, we compare our method with non-ensemble approaches.

\subsection{Distributionally robust optimization}
Distributionally Robust Optimization (DRO) \cite{DRO,GohDRO,bertsDRO} considers the case where the test distribution is different from the train distribution. It does that by defining an uncertainty set of test distributions around the training distribution and learning a classifier that minimizes the expected risk with respect to the worst-case test distribution. DRO has been shown to be equivalent to standard robust optimization (RO) \cite{xu2012distributional} and has a strong relationship with regularization methods \cite{DROLogistic}, Risk-Aversion methods and game theory \cite{DROreview}. For more details, see \cite{DROreview}. As far as we know, DRO theory was not applied for long-tail learning.

Many studies looked into learning representations that are robust to adversarial attacks \cite{bertsimas2019robust,arnab2018robustness,goswami2018unravelling} but this is outside the scope of the current paper.

\section{Overview of our approach}
We start with an overview of the main idea of our approach and provide the details in subsequent sections. 

Our goal is to learn a representation at the last layer of a deep network, such that the distributions of samples of different classes are well separated. Then, they can later be correctly classified by different linear classifiers. Deep networks can learn such representations efficiently when trained with enough labeled data. However, when a class has only few samples, the distribution of training samples may not represent well the true distribution of the data, and the representations learned by the model hinder the classifier.

To remedy this shortcoming, we design a loss that is applied to the representation, which takes into account errors due to a small number of samples. Figure \ref{fig:method} illustrates this idea. 
Our loss extends standard contrastive losses, which pull samples closer to the centroid of their own class and push away samples that belong to other classes. Our new loss accounts for the fact that the true centroids are unknown, and their estimate is noisy. It, therefore, optimizes against the worst possible centroids within a safety hyper ball around the empirical centroid.

In the general case, computing such a worst-case loss may be computationally hard. We further derive an upper bound of that loss that can be computed easily and show that using that bound as a surrogate loss, yields significantly better representations. 
The resulting loss has a simple general form, yielding that the loss for a sample $z$ is 
\begin{eqnarray}
    \mathcal{L}(z)_{Robust} = - 
    \log{\frac{e^{-\edistz-\Delta}} 
    {\sum_{z'}{e^{-\edistzp-\Delta'}}}}.
    \label{per-sample-loss}
\end{eqnarray}
where $\edistz$ measures the distance in feature space between a sample $z$ and the estimated centroid of its class $\ecentroid$, $\Delta$ and $\Delta'$ are robustness margins that we describe below.

\section{Distributional Robust Optimization}

When learning a classifier, we seek a model $f$ that minimizes the expected loss for the data distribution $P(x,y)$. 
\begin{equation}
      \min_{f}E_{x \sim P}[l_f(x)] = \int l(f(x),y) d P(x,y).
\end{equation}
Since the data distribution is not known, Empirical Risk Minimization (ERM) \cite{vapnik2013nature} proposes to use the training data for an  empirical estimate of the data distribution $P_{\delta} = \frac{1}{n}\sum \delta(x=x_i, y=y_i)$, where $\delta$ is the kronecker delta. 
\begin{equation}
     \textbf{ERM: \quad }\min_{f}E_{(x,y) \sim P_{\delta}}\left[l_f(x)\right] 
\end{equation}
Unfortunately, using $P_\delta$ to approximate $P$ makes the naive assumption that the test distribution would be close to the empirical train distribution. That assumption may be far from true when the training data is small. In those cases, it is beneficial to choose other estimates of $P$ that are more likely to reduce the loss over the test distribution.

One such solution is given by \textit{Distributionally robust optimization} (DRO) \cite{GohDRO, bertsDRO}. It suggests learning a model $f$ that minimizes the loss within a family of possible distributions 
\begin{equation}
    \textbf{DRO: \quad } \min_{f}\sup_{Q\in U}{E_{(x,y)\sim Q}[l_f(x)]}. 
\end{equation}
DRO aims to perform well simultaneously for a set of test distributions, each in an uncertainty set $U$. The set of distributions $U$ is typically chosen as a hyper ball of radius $\epsilon$ ($\epsilon$-ball) around the empirical training distribution $\widehat{P}$:
\begin{equation}    
    U := \{Q: D(Q,\widehat{P}) \leq \epsilon \} ,
\end{equation}
where $D$ is a discrepancy measure between distributions, usually chosen to be the KL-divergence or the Wasserstein distance.

\section{Our approach}

We now formally describe our approach: \textbf{DRO-LT Loss}.
\subsection{Preliminaries}
We are given $n$ labeled samples $(x_i,y_i)$ $i=1,\ldots,n$, where $y_i$ is one of $k$ classes $\{c_1,\ldots,c_k\}$.
Let $f_{\theta}$ be a feature extractor function $f$ with learnable parameters $\theta$, which maps any given input sample $x_i$ to $z_i = f_{\theta}(x_i)$. The set of mapped input $Z={z_1,...z_n}$ resides in some latent vector space.

Let $S_c=\{z_i|y_i=c\}$ be the set of feature vectors representing samples whose label is $c$. We denote by $\ecentroid$ the empirical centroid of that set, namely the mean of all feature vectors $z$ of class $c$, $\ecentroid = \frac{1}{|c|}\sum_{i=1}^{|c|}{z_i}$. We also denote by $\centroid$ the mean of the true data distribution $P$ of samples from class $c$, $\centroid = E_{x\sim P|y=c}[z]$.

\subsection{A loss for representation learning}
We adopt ideas from metric learning and representation learning \cite{weinberger2006distance,kulis2012metric,snell2017prototypical} for designing a representation-learning loss. 

Given a representation $z_i$ of a sample $x_i$ that has a label $y_i=c$, we wish to design a loss that maps $z_i$ close in feature space to other samples from the same class, and far from samples of other classes. 
We start by developing a contrastive loss and later extend it to a robust one.

Consider a sample $(x_i, y_i)$ from a class $y_i=c$, whose feature representation is $z_i$. 
We model the samples of class $c$ as if they are distributed around a centroid $\mu_c$, and the likelihood of a sample decays exponentially with the distance from the centroid $\mu_c$ of class $c$. Such exponential decay has been long used in metric learning and can be viewed as reflecting a random walk over the representation space \cite{goldberger2004neighbourhood, globerson2004euclidean}. The normalized likelihood of a sample is, therefore:    
\begin{equation}
    P(z_i|\centroid) = \frac{e^{-\distz}}{\sum_{z' \in Z}{e^{-\distzp}}},
    \label{norm_LL}
\end{equation}
where $\dist$ is a measure of divergence or distance between the centroid $\centroid$ and a feature representation $z_i$. $\dist$ is typically set to be the Euclidean distance, but can also be modeled with heavier tails as when using a student t-distribution \cite{vanDerMaaten2008}.

Similarly, for a set $S_c = \{z_i | y_i = c\}$, the log-likelihood is: 
\begin{equation}
    \log P(S_c|\centroid) = \sum_{z_i \in S_c}
    \log \frac{e^{-\distz}}
    {\sum_{z' \in Z}{e^{-\distzp}}}.
    \label{set_LL}
\end{equation}
We define a negative log-likelihood loss as a weighted average over per-class losses:
\begin{eqnarray}
    \mathcal{L_{NLL}} &=& \sum_{c \in C} w(c)
    \left(-\log P(S_c|\centroid)\right) \\ \nonumber
    &=& - \sum_{c \in C} w(c)
    \sum_{z_i \in S_c}
    \log{\frac{e^{-\distz}} 
    {\sum_{z' \in Z}{e^{-\distzp}}}},
    \label{NLL_loss}
\end{eqnarray}
where $w(c)$ are class weights. Setting $w(c) = \frac{1}{|S_c|}$ gives equal weighting to all classes and prevents head classes from dominating the loss.  

\subsection{A robust loss}
The log-likelihood loss of Eq. \eqref{set_LL} operates under the assumption that the centroid of each class $\centroid$ is known. In reality, it is not available to us. Following Empirical-risk minimization, we could naively plug in the empirical estimate $\ecentroid$ in Eq. \eqref{set_LL}, but $\ecentroid$ may be a poor approximation of $\centroid$, 
and estimation error grows as the number of samples $|S_c|$ decreases. As a result, the log-likelihood and the loss would also be badly estimated.

Instead of approximating the NLL loss ($-\log P(S_c|\centroid)$), we show that we can bound it with high probability by computing the worst-case loss over a set of candidate centroids. 
For that purpose, we take an approach based on distributionally robust optimization.

Let $\widehat{p_c}$ be the empirical  distribution of samples of class $c$. We define the uncertainty set of candidate distributions in its neighborhood to be 
\begin{equation}
    U_c := \{q | D(q || \widehat{p_c}) \leq \epsilon_{c}\},
\end{equation}
where $D$ is a measure of divergence of two  distributions. 
Specifically, we consider here the case where $D$ is the Kullback-Leibler divergence, and the distributions under consideration are same-variance spherical Gaussians. In this case, their divergence equals 
$%
    D_{KL}(q||\widehat{p_c}) = \frac{1}{2\sigma^2}d(\mu_q, \mu_{\widehat{p_c}} )^2
$%
\cite{Cover1991ElementsOI}, where $d$ is the Euclidean distance. Hence for any $q\in U_c$, we have 
\begin{equation}
    d(\mu_q, \mu_{\widehat{p_c}}) \leq  \sigma_c\sqrt{2\epsilon_c} \equiv \varepsilon_c
    \label{eq:bound},  
\end{equation}
where we define $\varepsilon_c\equiv\sigma_c\sqrt{2\epsilon_c}$ for convenience.

We now derive an upper bound on the NLL loss that we can compute using the estimated centroids $\ecentroid$. The form of the bound is very intuitive,  it can be viewed as a modification of the NLL loss, where the distance between a sample and an empirical centroid, is increased by a factor that depends on the radius of the uncertainty ball. 
\begin{theorem}
    Let $\epsilon_c$ be the radius of the uncertainty set $U_c$ and let $\sigma_c$ be the variance of the distribution of a sample in class $c$. Let $p(\epsilon)$ be the probability that the true distribution whose centroid is  $\centroid$ is within $U_c$.
    The per-class negative log-likelihood is bounded by 
    \begin{eqnarray}
        -\log P(z|\centroid)
        \le
         -\log \frac
               {e^{-\edistz-\epsc}}
               {\sum_{z' \in Z} e^{-\edistzp-\epsc\delta(z',c)} 
               }
    \end{eqnarray}
    with probability $p(\epsilon)$, where $\delta(z,c)=1$ if $z$ is of the class $c$ and 0 otherwise and $\varepsilon_{c} = \sigma_{c}\sqrt{2\epsilon_{c}}$.
\end{theorem}

\begin{proof}
Denote the negative log-likelihood of a given sample $z$ and class distribution $p$ with a centroid $\mu$ by $\mathcal{L}(z,p) = - \log P(z|\mu)$. With probability $p(\epsilon)$, the true distribution $p_c$ is within the uncertainty ball $p_c\in U_c$. Therefore   
\begin{equation}
    \mathcal{L}(z,p_c) 
    \leq \max_{q \in U_c}{\mathcal{L}(z,q)}
    = \max_{q\in U_c} (-\log P(z|\centroid^q)).
\end{equation}
Here, $\centroid^q$ is the centroid of a distribution $q\in U_c$. To bound $-\log P(z|\centroid^q)$, we use the triangle inequality and write: 
\begin{eqnarray}
    d(\centroid^q , z) &\leq& d(\ecentroid , z) + d(\ecentroid , \centroid^q) \\ \nonumber
    d(\ecentroid , z) &\leq& d(\centroid^q , z) + d(\ecentroid , \centroid^q) 
\end{eqnarray}
yielding
\begin{eqnarray}
    -d(\centroid^q,z) \!\!\!\!\!\! &\geq\!\! - d(\ecentroid,z) - d(\ecentroid,\centroid^q) \!\!\geq\!\! - d(\ecentroid,z) - \varepsilon_{c}
    \label{numerator_ineq}\\ %
    -d(\centroid^q,z) \!\!\!\!\!\! &\leq\!\! - d(\ecentroid , z) + d(\ecentroid , \centroid^q) \!\!\leq\!\! - d(\ecentroid , z) +  \varepsilon_{c}
\label{denominator _ineq}
\end{eqnarray}
Applying Eq.~\eqref{numerator_ineq} to the numerator of 
Eq. \eqref{norm_LL}, and applying Eq. \eqref{denominator _ineq} to its denominator, we obtain:
\begin{equation}
    P(z|\centroid^q) = {\frac{e^{-\qdistz}}
    {\sum_{z' \in Z}{e^{-\qdistzp}}}} \geq 
     \frac
        {e^{-\edistz-\epsc}}
        {\sum_{z' \in Z} e^{-\edistzp-\epsc\delta(z',c)}}
\end{equation}
where $\delta(z,c)=1$ when $c$ is the class of $z$ and $0$ otherwise. This inequality holds for all distributions $q\in U_c$, and also for the true distribution $\centroid$ as a special case (with probability $p(\epsilon)$). The negative log-likelihood is therefore bounded by
\begin{equation}
-\log P(z|\centroid) \leq -\log \frac
    {e^{-\edistz-\epsc}}
    {\sum_{z' \in Z} e^{-\edistzp-\epsc\delta(z',c)}}.
    \label{eq_upper}
\end{equation}
This completes the proof. 
\end{proof}

Based on the theorem, define a surrogate robustness loss:  
\begin{eqnarray}
    \mathcal{L}_{Robust} = - \sum_{c \in C} w(c)
    \sum_{z \in S_c}
    \log{\frac{e^{-\edistz-\epsc}} 
    {\sum_{z'}{e^{-\edistzp-\epsc\delta(z',c)}}}}
    \label{robust_loss}
\end{eqnarray}
This surrogate loss amends the loss of Eq. (8) in a simple and intuitive way. It increases the distance $\edistz$ between a sample and an empirical centroid in a way that depends on the radius of the uncertainty ball of a class.

\paragraph{Joint loss.}
In practice, we train the deep network with a combination of two losses. A standard cross-entropy loss is applied to the output of the classification layer, and the robust loss is applied to the latent representation of the penultimate layer. We linearly combine these two losses 
\begin{equation}
    \mathcal{L} = \lambda \mathcal{L}_{CE} + (1-\lambda)\mathcal{L}_{Robust}.
    \label{eq_tradeoff}
\end{equation}
The trade-off parameter $\lambda$ can be tuned using a validation set. See implementation details below. 

\paragraph{$p(\varepsilon)$ and Lower bound.} Appendix \ref{supp_p_epsilon} provides a formal definition of $p(\varepsilon)$. Appendix \ref{supp_lower_bound} derives a lower bound of our loss and also shows empirically that the bounds are tight in our experiments.

\section{Training}
\noindent \textbf{Uncertainty radius:} 
The size of the uncertainty $\epsilon$-ball around each class plays an important role.  When the uncertainty radius is too small, the probability that the true centroid is within the uncertainty area decreases, together with the probability that the bound holds. When the radius is too large, the bound is more likely to hold, but it becomes less tight. 
Furthermore, since tail classes have fewer samples, the estimate of class centroids is expected to be noisier and a larger radius is needed. 

We explored three ways to determine the radius:
\begin{enumerate}[noitemsep,topsep=0pt,leftmargin=10pt]
    \item \textbf{Shared $\varepsilon$:} All classes share the same uncertainty radius.
    \item \textbf{Sample count $\varepsilon / \sqrt{n}$:} The class radius scales with $1/\sqrt{n}$, where $n$ is the number of training samples.
This scaling is based on the fact that the standard-error-of-the-mean decays as $\sqrt{n}$, and leads to tail classes having a larger safety radius. See Appendix \ref{supp_sample_count} for more details.
    \item \textbf{Learned $\varepsilon$:} We treated the radius as a learnable parameter and tuned its value during training. 
\end{enumerate}

\noindent In the first two cases, the radius parameter $\varepsilon$ was treated as a hyper-parameter and tuned using a validation set.
~\newline

\vspace{-15pt}
\paragraph{Training process:} 
To compute the robustness loss, an initial feature representation of the data is required for estimating class centroids. As a result, during training, we first train the model with standard cross-entropy loss ($\lambda = 1$) to learn initial feature representations and centroids. Then, we add the \DRO{} loss to the training by setting $\lambda<1$. Finally, as in  \cite{Kang2019DecouplingRA}, we learn a balanced classifier by freezing the feature extractor and fine-tune only the classifier with balanced sampling.%

\vspace{-10pt}
\paragraph{Estimating centroids:}
Calculating the centroids for each class is computationally expensive if computed often using the full dataset. One could estimate the centroids within each batch, but with unbalanced data, minority classes would hardly have any samples and their centroid estimates would be very poor. 
Instead, we compute the features $z_i$ for every sample $x_i$ at the beginning of every epoch, compute the centroids for each class, and keep them fixed in memory for the duration of the epoch. 

\section{Experiments}

\subsection{Datasets}
We evaluated our proposed method using experiments on three major long-tailed recognition benchmarks.

\textbf{(1) CIFAR100-LT} \cite{cao2019learning}: A long-tailed version of CIFAR100 \cite{Krizhevsky2009LearningML}. CIFAR100 consists of 60K images from 100 categories.  Following \cite{cao2019learning}, we control the degree of data imbalance with an imbalance factor $\beta$. $\beta = \frac{N_{max}}{N_{min}}$, where $N_{max}$ and $N_{min}$ are the number of training samples for the most frequent and the least frequent classes, respectively. We conduct experiments with $\beta =$ 100, 50, and 10.
\newline \textbf{(2) ImageNet-LT} \cite{liu2019large}: A long-tailed version of the large-scale object classification dataset ImageNet \cite{Deng2009ImageNetAL} by sampling a subset following the Pareto distribution with power value $\alpha=6$. Consists of 115.8K images from 1000 categories with 1280 to 5 images per class.
\textbf{(3) iNaturalist} \cite{Horn2017TheIC}: A large-scale dataset for species classification. It is long-tailed by nature, with an extremely unbalanced label distribution. Has 437.5K images from 8,142 categories.

\subsection{Compared methods}
We compared our approach with the following methods. 

\noindent\textbf{\textit{(A) Baselines:}} \textbf{CE:} Naive training with a cross-entropy loss; \textbf{ReSample:} Over-sampling classes to reach a uniform distribution as in   \cite{cao2019learning}. \textbf{\textit{(B) Loss-manipulation:}} \textbf{Reweight:} Reweight the loss as in \cite{cao2019learning,zhou2020bbn}; \textbf{Focal Loss} \cite{Lin2017FocalLF} and \textbf{LDAM Loss} \cite{cao2019learning}. \textbf{\textit{(C) Two-stage fine-tuning:}} \textbf{$\tau$-norm} \cite{Kang2019DecouplingRA}, \textbf{cRT} \cite{Kang2019DecouplingRA} and \textbf{smDragon} \cite{samuel2020longtail}. \textbf{\textit{(D) Representation learning:}} SSL the method by  \cite{supervisedsimclr} and SSP, the method by \cite{yang2020rethinking}.  

\noindent\textbf{\textit{(E) DRO-LT variants:}} We compared four ways to set the radii of $U_c$.  Set $\varepsilon=0$ (ERM); $\varepsilon$ value shared across all classes; A shared value divided by the square root of class size ($\varepsilon/\sqrt{n}$);  and a learned value per class.

\begin{table}[t!]
    \centering
    \scalebox{0.8}{
    \setlength{\tabcolsep}{10pt}
    \begin{tabular}{l|ccc} 
    Imbalance Type & \multicolumn{3}{c}{\textbf{Long-tailed CIFAR-100}} \\
     \hline
     Imbalance Ratio & 100 & 50 & 10 \\ [0.5ex] 
      \hline
       CE*& 38.32 & 43.85 & 55.71 \\
     ReSample~\cite{Cui2019ClassBalancedLB} & 33.44 & - & 55.06  \\
     \hline
      ReWeight~\cite{Cui2019ClassBalancedLB} & 33.99 & - & 57.12 \\
     Focal Loss~\cite{Lin2017FocalLF} & 38.41 & 44.32 & 55.78 \\
     LDAM Loss~\cite{cao2019learning} & 39.60 & 44.97 & 56.91 \\
     \hline
     $\tau$-norm*~\cite{Kang2019DecouplingRA} & 41.11 & 46.74 & 57.06 \\
     cRT*~\cite{Kang2019DecouplingRA} & 41.24 & 46.83 & 57.93 \\
     smDRAGON~\cite{samuel2020longtail} & 43.55 & 46.85 & 58.01 \\
     \hline
       SSL* ~\cite{supervisedsimclr} & 37.51 & 44.02 & 56.70 \\
       SSP ~\cite{yang2020rethinking} & 43.43 & 47.11 & 58.91 \\
     \hline
    \textbf{DRO-LT (Ours)} &&& \\ $\varepsilon=0$ (ERM) & 43.92 & 52.31 & 59.54 \\
    Shared $\varepsilon$ & 45.66 & 55.32 & 61.22 \\
    $\varepsilon/\sqrt{n}$ & 46.92 & 57.20 & 63.10 \\
  Learned $\varepsilon$ & \textbf{47.31} & \textbf{57.57} & \textbf{63.41}\\
    \bottomrule
    \end{tabular}}
    \caption{Top-1 accuracy of ResNet32 on long-tailed CIFAR-100~\cite{cao2019learning}, comparing our method and SoTA techniques. Asterisks * denote reproduced results. DRO-LT variants achieve best results on all imbalance ratios. 
    }
    \label{cifar-bench}
\end{table}

\subsection{Evaluation Protocol}
Following \cite{liu2019large} and \cite{Kang2019DecouplingRA}, we report the top-1 accuracy over all classes on class balanced test sets. This metric is denoted by \textit{"Acc"}. For CIFAR-100 and ImageNet-LT, we further report accuracy on three splits of the set of classes. \textit{"Many":} classes with more than 100 train samples; \textit{"Med":} classes with 20 -- 100 train samples;  and \textit{"Few": } classes with less than 20 train samples. 

\subsection{Implementation details}
In all experiments, we use an SGD optimizer with a momentum of 0.9 to optimize the network. For long-tailed CIFAR-100, we follow \cite{cao2019learning} and train a ResNet-32 backbone on one GPU with a multistep learning-rate schedule. For ImageNet-LT and iNaturalist, we follow \cite{Kang2019DecouplingRA} and use the cosine learning-rate schedule to train a ResNet-50 backbone on 4 GPUs.

\textbf{Hyper-parameter tuning:} We determined the number
of training epochs (early-stopping), and tuned hyperparameters using the validation set.
We optimized the following hyper-parameters: (1) Radius parameter $\varepsilon \in \{1,2,5,10,30,70\}$ for "Shared $\varepsilon$" and "Sample Count $\varepsilon / \sqrt{n}$". (2) Trade-off parameter $\lambda \in {\{0,0.3,0.5,0.7,1\}}$. (3) Learning rate $\in {\{10^{-4},10^{-3},10^{-2}\}}$. we studied the sensitivity of the accuracy to the values of $\varepsilon$ and $\lambda$, and found that high accuracy is obtained for a wide range of $\epsilon$ values. We also found accuracy to be quite stable when tuning $\lambda$ and selected $\lambda=0.5$. See Appendix \ref{supp_lambda_trade_off} for details.

\begin{figure}[t!]
\centering 
    \includegraphics[width=0.66\columnwidth]{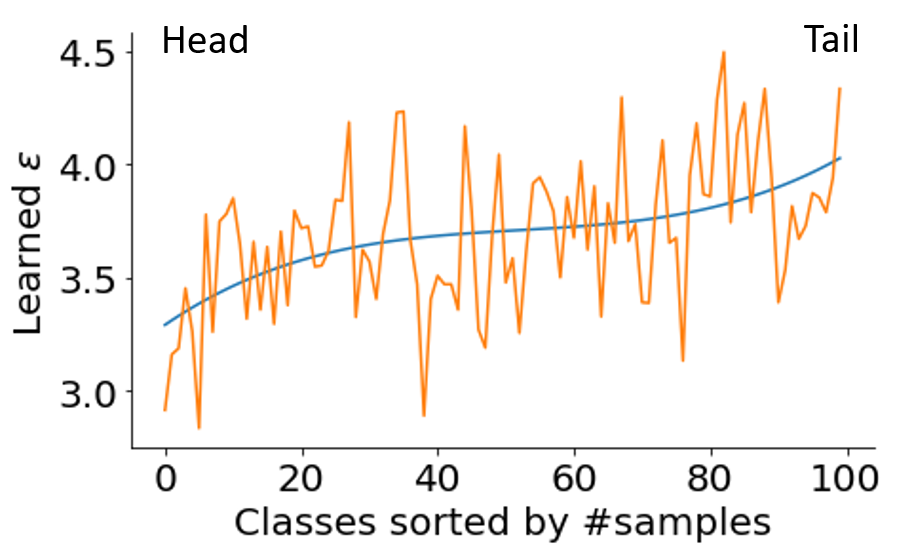}
    \caption{Learned $\varepsilon$ values for each class of a ResNet-32 trained on CIFAR100-LT with imbalance ratio 50. Blue: Polynomial fit.
    }
    \label{class-wise-epsilons}
\end{figure}

\begin{table*}[t!]
    \begin{center}
      \scalebox{0.8}{
    \setlength{\tabcolsep}{5pt} %
    \begin{tabular}{l|cccc|cccc} 

     & \multicolumn{4}{c|}{\textbf{Long-tailed CIFAR-100}} &
    \multicolumn{4}{c}{\textbf{Long-tailed ImageNet}} \\
 \hline
 Methods & Many & Med & Few & Acc & Many & Med & Few & Acc  \\ [0.5ex] 
  \hline
  CE*& \textbf{65.5} & 37.9 & 7.4 & 38.3 &
  \textbf{64.0} & 38.8 & 5.8 & 41.6 \\
 LDAM Loss~\cite{cao2019learning} & 61.0 & 41.6 & 19.8 & 39.6 & - & - & - & -  \\
 OLTR ~\cite{openlongtailrecognition} & 61.8 & 41.4 & 17.6 & 41.2 & - & - & - & - \\
 $\tau$-norm~\cite{Kang2019DecouplingRA} & 61.4 & 42.5 & 15.7 & 41.1 & 56.6 & 44.2 & 27.4 & 46.7\\
  smDragon ~\cite{samuel2020longtail} & 60.5 & 44.3 & 23.5 & 43.5 & 59.7 & 44.2 & 25.3 & 47.4\\
  SSL * ~\cite{supervisedsimclr} & 64.1 & 36.9 & 7.1 & 37.5 &
  61.4 & 47.0 & 28.2 & 49.8 \\
  \hline
  \textbf{DRO-LT (Ours)} &&&&&&& \\
  $\varepsilon=0$ (ERM) & 61.9 & 43.7 & 22.2 & 43.9 &
  61.0 & 45.5 & 26.8 & 48.1 \\
  Shared $\varepsilon$ & 64.1 & 47.9 & 21.5 & 45.7 $\pm$ 0.2&
  62.6 & 45.2 & 30.5 & 51.6 $\pm$ 0.4\\
  $\varepsilon/\sqrt{n}$ & 65.0 & 49.8 & 22.3 & 46.9 $\pm$ 0.2&
  63.8 & 49.5 & 32.7 & 53.0 $\pm$ 0.3\\
  Learnable $\varepsilon$ & 64.7 & \textbf{50.0} & \textbf{23.8} & \textbf{47.3 $\pm$ 0.1} &
  \textbf{64.0} & \textbf{49.8} & \textbf{33.1} & \textbf{53.5 $\pm$ 0.2} \\
    \bottomrule
    \end{tabular}}
    \end{center}
    \vspace{-15pt}
    \caption{Top-1 accuracy on long-tailed CIFAR-100~\cite{cao2019learning} with imbalance factor 100, and ImageNet-LT\cite{liu2019large}. We also report many-shot ("Many"), medium-shot ("Med") and few-show ("Few") performance separately. Our method performs well on tail classes without sacrificing head accuracy. Asterisks * denote results reproduced using code published by authors.}
    \label{cifar-imagenet-bench}
\end{table*}

\section{Results}
\label{sec:results}
\mypar{CIFAR100-LT: } Table \ref{cifar-bench} compares DRO-LT with common long-tail methods on CIFAR100-LT. It shows that all our robust loss variants consistently achieve the best results on all imbalance factors, emphasizing the importance of robust learning in unbalanced data. "Learned $\varepsilon$" outperforms other variants of our method. \textbf{CIFAR-10-LT:} Appendix  \ref{supp_cifar_10} provides results for CIFAR-10-LT with imbalanced ratio 100. 
\textbf{ImageNet-LT: } 
Table \ref{cifar-imagenet-bench} further evaluates our approach on ImageNet-LT and CIFAR100-LT (imbalance factor=100) reporting accuracy for different test splits. DRO-LT performs well on tail classes ("Few") as well as head classes ("Many"). This contrasts with previous methods that sacrifice head accuracy for better tail classification. This also suggests that DRO-LT learns high-quality features for all classes. 
\textbf{iNaturalist:} Table \ref{iNaturalist-bench} evaluates our approach on the large-scale iNaturalist. DRO-LT slightly improves the accuracy compared with SoTA baselines. See Appendix \ref{supp_inat} for more results and further analysis.

\begin{SCtable}
      \scalebox{0.8}{
    \setlength{\tabcolsep}{5pt} 
    \begin{tabular}{l|c}
    {} & \textbf{iNaturalist}\\
    \midrule
    CE & 61.7   \\
    LDAM Loss~\cite{cao2019learning} & 68.0  \\
    $\tau-$norm~\cite{Kang2019DecouplingRA} & 69.5 \\
    CB LWS~\cite{Kang2019DecouplingRA} & 69.5   \\
    smDragon~\cite{samuel2020longtail} & 69.1 \\
    SSL *~\cite{supervisedsimclr} & 66.4 \\
\midrule
    \textbf{DRO-LT (ours)} \\
    Learned $\varepsilon$ & \textbf{69.7 $\pm$ 0.1}
    \\
    \bottomrule
    \end{tabular}}
\caption{Top-1 accuracy on long-tailed iNaturalist. DRO-LT achieves slightly better results compared to previous methods. Asterisks * denote reproduced results by us.}
\label{iNaturalist-bench}
\end{SCtable}

\mypar{Classifier vs feature extractor:} Our method focuses on improving the learned representation at the penultimate layer. Other methods focused on improving the classifier applied to that representation. Therefore, it is interesting to explore the relation between these two tasks (compare with  \cite{zhou2020bbn,Kang2019DecouplingRA}). We, therefore, compared different feature extractors and classifier training methods.

\ignore{
\begin{figure*}[t!]
    \centering
    \includegraphics[width=0.60\columnwidth]{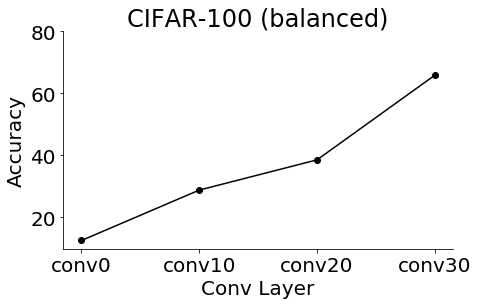}
    \includegraphics[width=0.60\columnwidth]{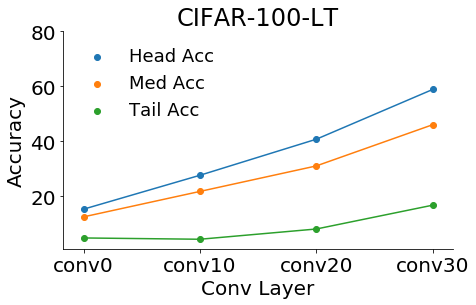}
    \vspace{-10pt}
    \caption{
    Accuracy of a nearest-centroid neighbor classifier when applied to convolutional layers 0, 10, 20, and 30 of a ResNet-32.
    \textbf{Left}: The model was trained on balanced CIFAR-100. The validation accuracy grows when using higher layers.
    \textbf{Right:} The model was trained on CIFAR-100-LT. We report balanced validation accuracy for head classes (blue), medium classes (orange) and tail classes (green). The accuracy gap between head and tail classes is substantial even at all layers.}
    \label{fig:nn-imbalance}
\end{figure*}
}

\begin{figure}[t!]
    \centering
    \includegraphics[width=0.7\columnwidth]{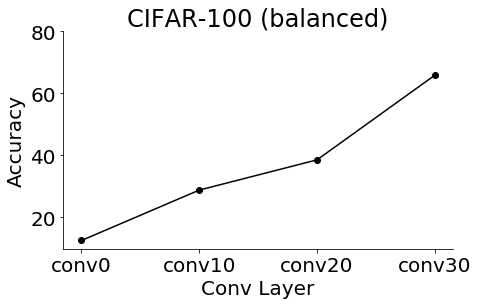}
    \includegraphics[width=0.7\columnwidth]{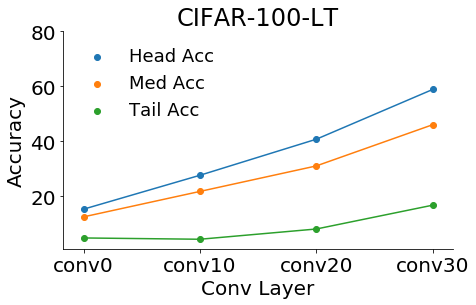}
    \vspace{-10pt}
    \caption{
    Accuracy of a nearest-centroid neighbor classifier when applied to convolutional layers 0, 10, 20, and 30 of a ResNet-32.
    \textbf{Top}: The model was trained on balanced CIFAR-100. The validation accuracy grows when using higher layers.
    \textbf{Bottom:} The model was trained on CIFAR-100-LT. We report balanced validation accuracy for head classes (blue), medium classes (orange) and tail classes (green). The accuracy gap between head and tail classes is substantial even at all layers.}
    \label{fig:nn-imbalance}
\end{figure}

For representation learning, we employ plain training with a cross-entropy loss (\textbf{CE}), balanced resampling of the data (\textbf{RS}) and our method (\textbf{DRO-LT}). For classifier learning, we freeze the parameters of the feature extractor and fine-tune the classifier in three ways: cross-entropy loss (\textbf{CE}), re-sampling (\textbf{RS}) following the protocol of \cite{cao2019learning}, and balanced classifier (\textbf{LWS}) \cite{Kang2019DecouplingRA}. Table \ref{tab:rep-clf} provides the top-1 accuracy of all combinations. It shows that our representation learning approach enables all types of classifiers to reach good performance, compared to other representation learning approaches. This strongly suggests that our method learns: (1) good feature representations for both head and tail classes, and (2) to separate them in a way that discriminative classifiers can easily distinguish between classes. 

\begin{table}[t!]
    \begin{center}
        \scalebox{0.8}{
        \setlength{\tabcolsep}{5pt}
        \begin{tabular}{l|c|c|c} 
        & \multicolumn{3}{c}{Classifier learning} \\
        \hline
        Representation  &  &   &  \\
        learning & CE & RS (cRT) & LWS \cite{Kang2019DecouplingRA} \\ [0.5ex] 
        \hline
         CE & 38.5 & 41.2 & 41.4  \\
         RS & 34.9 & 37.6 & 36.3 \\
         DRO-LT (\textbf{ours}) & \textbf{41.2} & \textbf{47.3} & \textbf{46.8}\\
        \bottomrule
        \end{tabular}}
    \end{center}
    \vspace{-15pt}
    \caption{Top-1 accuracy of different representation learning manners and classifier learning manners, on CIFAR100-LT. CE refers to Cross-Entropy, RS refers resampling and LWS refers to balanced classifier according to \cite{Kang2019DecouplingRA}. The results suggest that simple resampling methods achieve good results when the learned features are good for both head and tail.}
  \label{tab:rep-clf}
\end{table}

\begin{figure}[t]
    \centering
        \includegraphics[width=0.75\columnwidth]{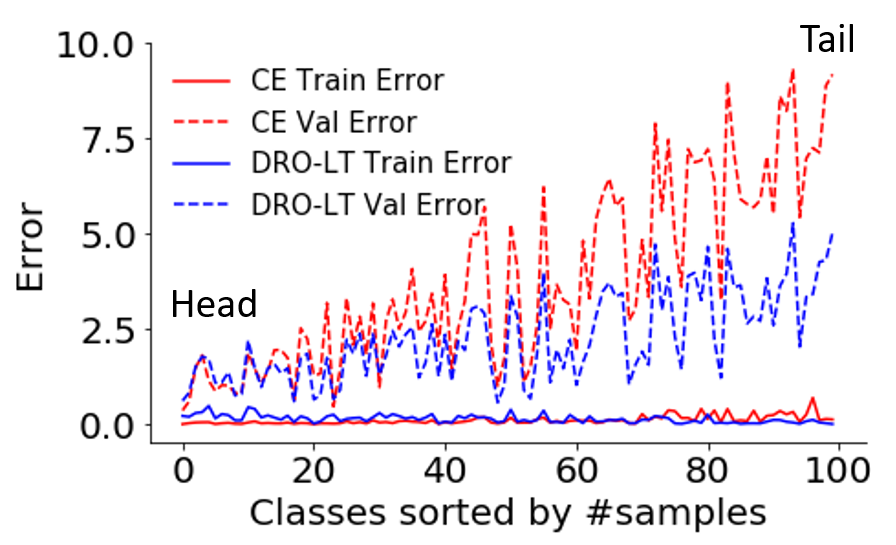}
        \vspace{-10pt}
    \caption{Comparing train and test error between a model trained with vanilla cross-entropy loss and a model trained with our DRO-LT. The gap between the train error and test error for tail classes is much lower in DRO-LT compared with vanilla, while remaining the same for head classes. 
    }
    \label{class-wise-error}
\end{figure}

\mypar{Adaptive robustness:} Figure \ref{class-wise-epsilons} shows the values of uncertainty per-class radii ($\varepsilon$) that were learned by training a ResNet-32 on CIFAR100-LT with an imbalance ratio of 50. The model learned slightly larger radii for tail classes compared with head classes. See the supplemental for details about the effect of the radius on accuracy.

\mypar{Robustness:} Our loss is expected to improve recognition mostly at tail classes. Figure \ref{class-wise-error} compares train error and test error between a model trained with cross-entropy loss (red) and a model trained with our approach (blue). Using a robustness loss cuts down errors substantially, in tail classes, without hurting head classes. 

\mypar{Feature space visualization:} To gain additional insight, we look at the t-SNE projection of learned representations and compared vanilla cross-entropy loss with our proposed method. Figure \ref{tsne-vis} shows that our learned feature space is more compact with margins around head and tail classes. Tail classes have larger margins since the estimation of their features is less accurate.

\mypar{How unbalanced are latent representations?} 
The above analysis focused on correcting the representation at the penultimate layer. But how biased are the representations at lower layers?
Intuitively, the low layers represent more "physical" properties of inputs, while higher layers capture properties that correspond more to semantic classes. 
One would expect that early layers would be quite balanced.

We tested class imbalance in several layers of a ResNet-32 in the following way. We first trained a ResNet-32 on the unbalanced CIFAR100-LT dataset. Then, for each latent representation, we computed the centroids of each class and used them to classify all samples using a nearest-centroid classifier. Figure \ref{fig:nn-imbalance} shows the accuracy obtained with this classifier when applied to layers 0, 10, 20, and 30 of the ResNet-32. When training with balanced data (up), the accuracy grows when using higher layers, as expected. Surprisingly, when training with unbalanced data (down), there is a substantial difference in accuracy achieved for head and tail classes in every layer that we tested. While accuracy grows when using higher layers, accuracy differences between head and tail classes are maintained, even in low layers that are believed to represent class agnostic features. In Appendix \ref{supp_centroid_nn}, we compare the accuracy of a nearest-centroid neighbor classifier between a model trained with standard cross-entropy loss and one trained with DRO-LT. We show that DRO-LT narrows the accuracy gap between head and tail.

\begin{figure}[t!]
    \includegraphics[width=0.73\linewidth]{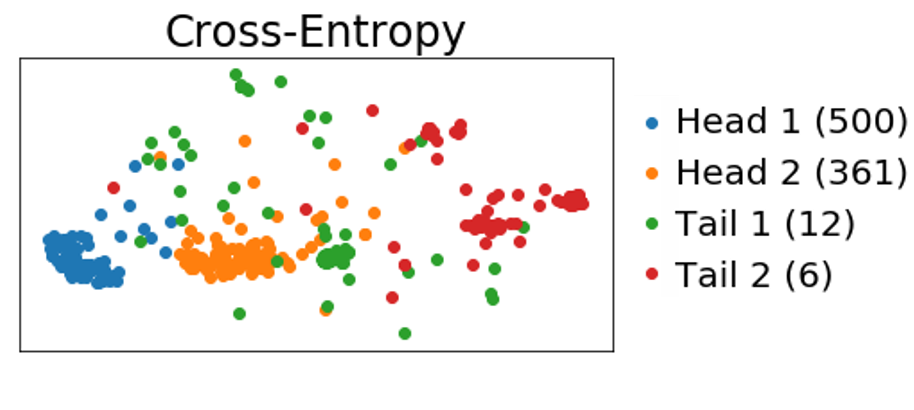}
    \includegraphics[width=0.51\linewidth]{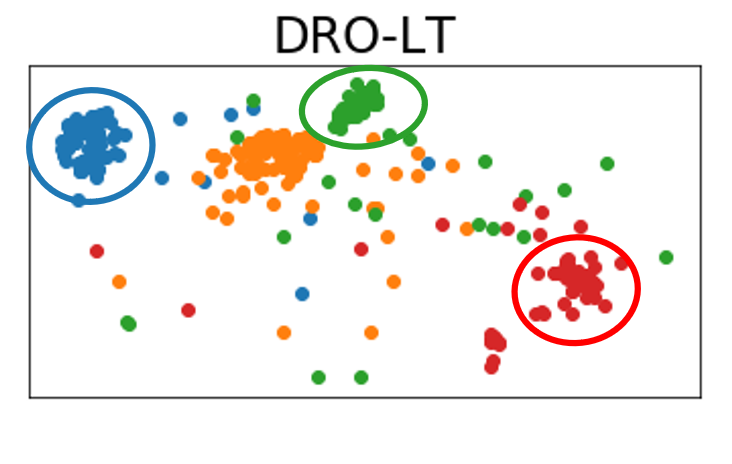}
    \vspace{-15pt}
    \caption{t-SNE visualization of embedding space of CIFAR100-LT obtained using cross-entropy loss and DRO-LT Loss method. The feature embedding of our model is more compact for both head (blue) and tail (green and red) classes and better separated. The number of samples for each class is written in parentheses.}
    \label{tsne-vis}
\end{figure}

\vspace{-5pt}
\section{Discussion}
\vspace{-5pt}
This paper investigates the feature representations learned by deep models trained on long-tail data. We find that such models suffer greatly from bias towards head classes in their feature extractor (backbone), which hurts recognition. This is in contrast to previous studies suggesting that unbalanced data does not hurt representation learning and re-balancing the classifier layer is sufficient. %
To learn a balanced representation, we take a robustness approach and develop a novel loss based on Distributionally Robust Optimization (DRO) theory. We further derive an upper bound of that loss that can be minimized efficiently. 
We show how the robustness safety margin can be learned during training, and do not require additional hyper-parameter tuning.

Training with a combination of the DRO-LT loss and a standard classifier sets new state-of-the-art results on three long-tail benchmarks: CIFAR100-LT, ImageNet-LT, and iNaturalist.
Our method not only improves the performance of tail classes but also maintains high accuracy at the head. %
These results suggest that proper training of representations for unbalanced data can have a large impact on downstream accuracy. 
We believe that our finding not only contributes to a deeper understanding of the long-tailed recognition task but can offer inspiration for future work.

~\newline 
\noindent
\small
\textbf{Acknowledgments:}
This work was funded by the Israeli innovation authority through the AVATAR consortium;  by the Israel Science Foundation (ISF grant 737/2018); and by an equipment grant to GC and Bar Ilan University (ISF grant 
2332/18).

\newpage

{\small
\bibliographystyle{ieee_fullname}
\bibliography{robust}

\begin{thebibliography}{10}\itemsep=-1pt

\bibitem{arnab2018robustness}
Anurag Arnab, Ondrej Miksik, and Philip~HS Torr.
\newblock On the robustness of semantic segmentation models to adversarial
  attacks.
\newblock In {\em CVPR}, 2018.

\bibitem{Beery2019SyntheticEI}
S. Beery, Y. Liu, D. Morris, J. Piavis, A. Kapoor, M. Meister, and P. Perona.
\newblock Synthetic examples improve generalization for rare classes.
\newblock {\em Preprint arXiv:1904.05916}, 2019.

\bibitem{bertsimas2019robust}
Dimitris Bertsimas, Jack Dunn, Colin Pawlowski, and Ying~Daisy Zhuo.
\newblock Robust classification.
\newblock {\em INFORMS Journal on Optimization}, 2019.

\bibitem{bertsDRO}
Dimitris Bertsimas, Vishal Gupta, and Nathan Kallus.
\newblock Data-driven robust optimization.
\newblock {\em Mathematical Programming}, 2018.

\bibitem{buda2017systematic}
M. Buda, A. Maki, and M. Mazurowski.
\newblock A systematic study of the class imbalance problem in convolutional
  neural networks.
\newblock {\em Neural Networks}, 2018.

\bibitem{cao2019learning}
K. Cao, C. Wei, A. Gaidon, N. Arechiga, and T. Ma.
\newblock Learning imbalanced datasets with label-distribution-aware margin
  loss.
\newblock {\em NeurIPS}, 2019.

\bibitem{chawla2002smote}
Nitesh~V Chawla, Kevin~W Bowyer, Lawrence~O Hall, and W~Philip Kegelmeyer.
\newblock Smote: synthetic minority over-sampling technique.
\newblock {\em Journal of artificial intelligence research}, 2002.

\bibitem{chu2020feature}
Peng Chu, Xiao Bian, Shaopeng Liu, and Haibin Ling.
\newblock Feature space augmentation for long-tailed data.
\newblock {\em arXiv preprint arXiv:2008.03673}, 2020.

\bibitem{Cover1991ElementsOI}
T. Cover and J. Thomas.
\newblock Elements of information theory.
\newblock 1991.

\bibitem{Cui2019ClassBalancedLB}
Y. Cui, M. Jia, T. Lin, Y. Song, and S. Belongie.
\newblock Class-balanced loss based on effective number of samples.
\newblock {\em CVPR}, 2019.

\bibitem{Deng2009ImageNetAL}
J. Deng, W. Dong, R. Socher, L. Li, K. Li, and F. Li.
\newblock Imagenet: A large-scale hierarchical image database.
\newblock {\em CVPR}, 2009.

\bibitem{Drummond2003C4}
Chris Drummond.
\newblock C 4 . 5 , class imbalance , and cost sensitivity : Why under-sampling
  beats oversampling.
\newblock 2003.

\bibitem{globerson2004euclidean}
Amir Globerson, Gal Chechik, Fernando~CN Pereira, and Naftali Tishby.
\newblock Euclidean embedding of co-occurrence data.
\newblock In {\em NIPS}, 2004.

\bibitem{GohDRO}
Joel Goh and Melvyn Sim.
\newblock Distributionally robust optimization and its tractable
  approximations.
\newblock {\em Operations research}, 2010.

\bibitem{goldberger2004neighbourhood}
Jacob Goldberger, Geoffrey~E Hinton, Sam Roweis, and Russ~R Salakhutdinov.
\newblock Neighbourhood components analysis.
\newblock {\em Advances in neural information processing systems}, 2004.

\bibitem{goswami2018unravelling}
Gaurav Goswami, Nalini Ratha, Akshay Agarwal, Richa Singh, and Mayank Vatsa.
\newblock Unravelling robustness of deep learning based face recognition
  against adversarial attacks.
\newblock In {\em AAAI}, 2018.

\bibitem{Han2005BorderlineSMOTEAN}
H. Han, W. Wang, and B. Mao.
\newblock Borderline-smote: A new over-sampling method in imbalanced data sets
  learning.
\newblock In {\em ICIC}, 2005.

\bibitem{He_2009_IEEE}
H. {He} and E.~A. {Garcia}.
\newblock Learning from imbalanced data.
\newblock {\em IEEE Transactions on Knowledge and Data Engineering}, 2009.

\bibitem{Horn2017TheIC}
G. Horn, O. Aodha, Y. Song, A. Shepard, H. Adam, P. Perona, and S. Belongie.
\newblock The inaturalist challenge 2017 dataset.
\newblock {\em ArXiv}, 2017.

\bibitem{hu2020learning}
Xinting Hu, Yi Jiang, Kaihua Tang, Jingyuan Chen, Chunyan Miao, and Hanwang
  Zhang.
\newblock Learning to segment the tail.
\newblock In {\em CVPR}, 2020.

\bibitem{Kang2019DecouplingRA}
B. Kang, S. Xie, M. Rohrbach, M. Yan, A. Gordo, J. Feng, and Y. Kalantidis.
\newblock Decoupling representation and classifier for long-tailed recognition.
\newblock {\em ICLR}, 2020.

\bibitem{khan2019striking}
Salman Khan, Munawar Hayat, Syed~Waqas Zamir, Jianbing Shen, and Ling Shao.
\newblock Striking the right balance with uncertainty.
\newblock In {\em CVPR}, 2019.

\bibitem{supervisedsimclr}
Prannay Khosla, Piotr Teterwak, Chen Wang, Aaron Sarna, Yonglong Tian, Phillip
  Isola, A. Maschinot, Ce Liu, and Dilip Krishnan.
\newblock Supervised contrastive learning.
\newblock {\em Nuerips}, abs/2004.11362, 2020.

\bibitem{Kim2020M2mIC}
Jaehyung Kim, Jongheon Jeong, and Jinwoo Shin.
\newblock M2m: Imbalanced classification via major-to-minor translation.
\newblock {\em CVPR}, 2020.

\bibitem{Krizhevsky2009LearningML}
Alex Krizhevsky.
\newblock Learning multiple layers of features from tiny images.
\newblock 2009.

\bibitem{kulis2012metric}
Brian Kulis et~al.
\newblock Metric learning: A survey.
\newblock {\em Foundations and trends in machine learning}, 2012.

\bibitem{Lin2017FocalLF}
T. Lin, P. Goyal, R. Girshick, K. He, and P. Doll{\'a}r.
\newblock Focal loss for dense object detection.
\newblock {\em ICCV}, 2017.

\bibitem{DRO}
Anqi Liu and Brian~D. Ziebart.
\newblock Robust classification under sample selection bias.
\newblock In {\em NIPS}, 2014.

\bibitem{liu2020deep}
Jialun Liu, Yifan Sun, Chuchu Han, Zhaopeng Dou, and Wenhui Li.
\newblock Deep representation learning on long-tailed data: A learnable
  embedding augmentation perspective.
\newblock In {\em CVPR}, 2020.

\bibitem{openlongtailrecognition}
Z. Liu, Z. Miao, X. Zhan, J. Wang, B. Gong, and S. Yu.
\newblock Large-scale long-tailed recognition in an open world.
\newblock In {\em CVPR}, 2019.

\bibitem{liu2019large}
Ziwei Liu, Zhongqi Miao, Xiaohang Zhan, Jiayun Wang, Boqing Gong, and Stella~X
  Yu.
\newblock Large-scale long-tailed recognition in an open world.
\newblock In {\em CVPR}, 2019.

\bibitem{Ouyang2016FactorsIF}
Wanli Ouyang, X. Wang, Cong Zhang, and X. Yang.
\newblock Factors in finetuning deep model for object detection with long-tail
  distribution.
\newblock {\em CVPR}, 2016.

\bibitem{DROreview}
H. Rahimian and S. Mehrotra.
\newblock Distributionally robust optimization: A review.
\newblock {\em ArXiv}, 2019.

\bibitem{Ryou_2019_ICCV}
S. Ryou, S. Jeong, and P. Perona.
\newblock Anchor loss: Modulating loss scale based on prediction difficulty.
\newblock In {\em ICCV}, 2019.

\bibitem{samuel2020longtail}
Dvir Samuel, Yuval Atzmon, and Gal Chechik.
\newblock From generalized zero-shot learning to long-tail with class
  descriptors.
\newblock In {\em WACV}, 2021.

\bibitem{DROLogistic}
Soroosh Shafieezadeh-Abadeh, Peyman~Mohajerin Esfahani, and D. Kuhn.
\newblock Distributionally robust logistic regression.
\newblock In {\em NIPS}, 2015.

\bibitem{snell2017prototypical}
Jake Snell, Kevin Swersky, and Richard~S Zemel.
\newblock Prototypical networks for few-shot learning.
\newblock {\em NIPS}, 2017.

\bibitem{vanDerMaaten2008}
Laurens van~der Maaten and Geoffrey Hinton.
\newblock Visualizing data using (t-sne).
\newblock {\em Journal of Machine Learning Research}, 2008.

\bibitem{vapnik2013nature}
Vladimir Vapnik.
\newblock {\em The nature of statistical learning theory}.
\newblock Springer science \& business media, 2013.

\bibitem{ride}
Xudong Wang, Long Lian, Zhongqi Miao, Ziwei Liu, and Stella~X Yu.
\newblock Long-tailed recognition by routing diverse distribution-aware
  experts.
\newblock {\em arXiv preprint arXiv:2010.01809}, 2020.

\bibitem{wang2019dynamic}
Yiru Wang, Weihao Gan, Jie Yang, Wei Wu, and Junjie Yan.
\newblock Dynamic curriculum learning for imbalanced data classification.
\newblock In {\em ICCV}, 2019.

\bibitem{weinberger2006distance}
Kilian~Q Weinberger, John Blitzer, and Lawrence~K Saul.
\newblock Distance metric learning for large margin nearest neighbor
  classification.
\newblock In {\em Advances in neural information processing systems}, 2006.

\bibitem{LFME}
Liuyu Xiang and G. Ding.
\newblock Learning from multiple experts: Self-paced knowledge distillation for
  long-tailed classification.
\newblock In {\em ECCV}, 2020.

\bibitem{xu2012distributional}
Huan Xu, Constantine Caramanis, and Shie Mannor.
\newblock A distributional interpretation of robust optimization.
\newblock {\em Mathematics of Operations Research}, 2012.

\bibitem{yang2020rethinking}
Yuzhe Yang and Zhi Xu.
\newblock Rethinking the value of labels for improving class-imbalanced
  learning.
\newblock {\em NeurIPS}, 2020.

\bibitem{zhou2020bbn}
Boyan Zhou, Quan Cui, Xiu-Shen Wei, and Zhao-Min Chen.
\newblock Bbn: Bilateral-branch network with cumulative learning for
  long-tailed visual recognition.
\newblock In {\em CVPR}, 2020.

\end{thebibliography}
}

\clearpage

\appendix

{\Large{\textbf{Supplemental Material}}}

\section{Additional information about  $p(\epsilon)$}
\label{supp_p_epsilon}

We provide a more formal definition of the probability $p(\epsilon)$. 
Consider the empirical distribution $\widehat{p_c(z)}$ of $n$ samples $z_1,.\ldots,z_n$ from class $c$. This empirical distribution $\hat{p}$ can be viewed by itself as a random variable because it receives different values for different instantiations of the random sample. $\hat{p}$ is distributed over the simplex $\Delta^d$ where $d$ is the dimension of the representation of $z$. As a result, the $\epsilon$-ball around $\ecentroid$ is also a random variable, since different random samples yield different centroids and balls. 

For any given $\epsilon$, some of these balls may cover the true distribution $p$ and some may not, depending on the instances $z_i$ drawn.  $p(\epsilon)$ denotes the probability that such an $\epsilon$-ball covers the true distribution. 

\section{Tightness of the bound} 
\label{supp_lower_bound}
We derive a lower bound of our loss and evaluate empirically that the bounds are tight.

The lower bound is easily derived in a very similar way to the derivation of the upper bound in Theorem 1 using the triangle inequality. The bound has the following form

\begin{equation}
-\log P(z|\centroid) \geq -\log \frac
    {e^{-\edistz+\epsc}}
    {\sum_{z' \in Z} e^{-\edistzp+\epsc\delta(z',c)}}
    \label{eq_lower}
\end{equation}

We further estimated empirically the relative magnitude of the bound gap 
$\frac{|Eq.\ref{eq_upper} - Eq.\ref{eq_lower}|}{Eq.\ref{robust_loss}}$, averaged over all samples. For CIFAR100-LT it was $0.07$, suggesting that the bounds are quite tight in our case.

\section{DRO-LT Sample count $\epsilon / \sqrt{n}$}
\label{supp_sample_count}

Given a set of samples $z_i$, their mean is known to have a standard deviation of $\sigma/\sqrt{n}$, where $\sigma$ is the standard deviation of the sample distribution $p(z)$. 
In our case, $\sigma$ is not known. 

The DRO-LT variant that we call  \textbf{Sample count} $\epsilon / \sqrt{n}$, can be viewed as assuming that all classes share the same standard deviation of their sample distribution $\sigma$, which we tune as a hyperparameter, and the uncertainty about class centroid only varies by the number of samples. 

\section{Loss trade-off parameter $\lambda$ }
\label{supp_lambda_trade_off}

Figure \ref{fig:lambda_ablation} quantifies the effect of the trade-off parameter  $\lambda$ (Eq.~\ref{eq_tradeoff}) on the validation accuracy. The model was trained on CIFAR100-LT with an imbalance factor of 100 and with our DRO-LT loss (\textbf{Learnable} $\varepsilon$ variant). It shows that training with DRO-LT alone ($\lambda = 0$) is not enough and leads to poor accuracy. Combining the robustness loss with a discriminative loss (cross-entropy) gives the best results and suggests high-quality feature representations and a discriminative classifier.

\section{Robustness vs Performance}
\label{supp_robustness_vs_performance}

\begin{figure}[t!]
\centering 
    \includegraphics[width=0.95\columnwidth]{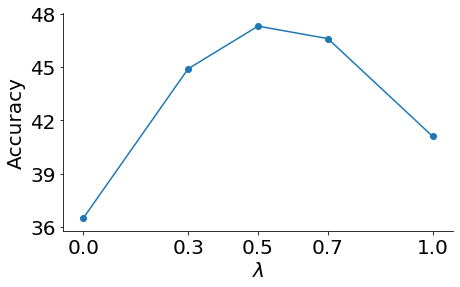}
    \caption{
    Validation accuracy of a model trained on CIFAR100-LT (imbalance factor 100) with different loss trade-off parameter ($\lambda$) values between standard cross-entropy loss and our DRO-LT loss (Learnable $\varepsilon$ variant).
    }
    \label{fig:lambda_ablation}
\end{figure}

\begin{figure}[t!]
\centering 
    \includegraphics[width=0.95\columnwidth]{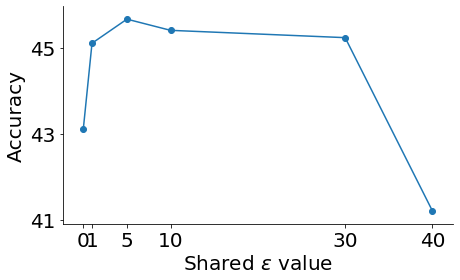}
    \caption{
    Validation accuracy of a model trained on CIFAR100-LT (imbalance factor 100) with different uncertainty radius $\varepsilon$. The radius is set to be equal for all classes in our \textbf{shared} $\varepsilon$ variant. 
    }
    \label{fig:shared_epsilon_ablation}
\end{figure}

Figure \ref{fig:shared_epsilon_ablation} explores the effect of different uncertainty radii ($\varepsilon$) on the validation accuracy of a model trained on CIFAR100-LT with an imbalance factor 100. The radius is set to be equal for all classes in our DRO-LT loss (\textbf{Shared} $\varepsilon$ variant). It shows that the accuracy is maintained over a large range of $\varepsilon$ values. Setting $\varepsilon$ to very small values nullifies the robustness and reduces accuracy.
At the same time, very large values of $\varepsilon$ cause the worst-case centroid in the $\epsilon$-ball to be too far from $\centroid$ making the bound too loose and again reduces the accuracy. 

\begin{figure*}[t!]
\centering 
    \includegraphics[width=0.95\columnwidth]{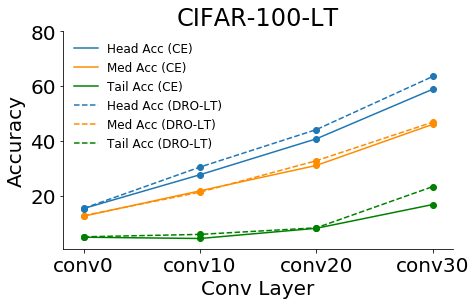}
    \caption{Accuracy of a nearest-centroid neighbor classifier when applied to convolutional layers 0, 10, 20, and 30 of a ResNet-32. Our model narrows the accuracy gap between head and tail classes.
    The model was trained on CIFAR-100-LT with an imbalance factor of 100. We compare a model trained with standard cross-entropy loss (solid line) and a model trained with DRO-LT (dashed line) where the loss is applied to the last convolutional layer (conv30). We report balanced validation accuracy for head classes (blue), medium classes (orange), and tail classes (green).
    }
    \label{fig:bias_with_ours}
\end{figure*}

\section{More analysis on iNaturalist:}
\label{supp_inat}

\begin{table}[t!]
    \begin{center}
      \scalebox{1}{
    \setlength{\tabcolsep}{3pt} %
\begin{tabular}{l|cccc}
    \textbf{iNaturalist} & Many & Med & Few & Acc \\
    \midrule
    CE*& 72.2 & 63.0 & 57.2 & 61.7 \\
    CB LWS & 71.0 & 69.8 & 68.8 & 69.5 \\
    \hline
     \textbf{DRO-LT (ours)} \\
    Shared $\varepsilon$ & 78.2 & 70.6 & 64.7 & 69.0 $\pm$ 0.2 \\
  $\varepsilon/\sqrt{n}$ & 71.0 & 68.9 & \textbf{69.3} & 69.1 $\pm$ 0.2 \\
  Learned $\varepsilon$& \textbf{73.9} & \textbf{70.6} & 68.9 & \textbf{69.7 $\pm$ 0.1} 
    \\
    \bottomrule
    \end{tabular}
    }
    \end{center}
    
    \caption{Top-1  accuracy on long-tailed iNaturalist with accuracy broken to Many-shot, Medium-shot and Few-shot classes. Our approaches improves the performance on head and tail classes.}
    \label{baselines_bench_inat}
\end{table}

Table \ref{baselines_bench_inat} compares DRO-LT with common long-tail methods on iNaturalist with accuracy
broken by class frequency: many-shot (”Many”), medium-shot (”Med”) and few-show (”Few”). It shows that
improvement is larger at the head (Many) and tail (Few), but
relatively small for most of the classes in this dataset (Med). We also provide the variance for 10 runs. 

\section{Results for CIFAR-10-LT:}
\label{supp_cifar_10}

Table \ref{baselines_bench} compares our approach with common long-tail methods on CIFAR-10-LT~\cite{cao2019learning}.  Our method outperforms all baselines.

\section{Balancing latent representations}
\label{supp_centroid_nn}

Here, we provide more analysis on the imbalance of latent representation and compare our approach with a standard baseline.

Figure \ref{fig:bias_with_ours} shows the accuracy obtained with a nearest-centroid classifier when applied to layers 0, 10, 20, and 30 of a ResNet-32. We compare a model trained with standard cross-entropy loss (solid line) and a model trained with DRO-LT (dashed line) where the loss is applied to the last convolutional layer (conv30). We show that our model narrows the accuracy gap between head classes (blue) and tail classes (green), mostly in the last layer. This shows the effectiveness of our approach on the latent representations of the deep models and suggests that applying our loss to the rest of the layers might result in a more balanced model overall.

\begin{table}[t!]
\begin{center}
      \scalebox{0.9}{
    \setlength{\tabcolsep}{5pt} 
    \begin{tabular}{l|c}
    \textbf{CIFAR-10-LT} & Acc \\
    \midrule
    CE & 72.1   \\
    CB LWS \cite{Kang2019DecouplingRA} & 73.5   \\
    LDAM DRW \cite{cao2019learning} & 77.03 \\
    smDragon \cite{samuel2020longtail} & 79.6 \\
\midrule
    Learned $\varepsilon$ (ours) & \textbf{82.6}
    \\
    \bottomrule
    \end{tabular}}
    \end{center}
\caption{Top-1 accuracy on long-tailed CIFAR-10 \cite{cao2019learning} with imbalance factor 100.}
    \label{baselines_bench}
\end{table}

\end{document}